\documentclass{article}
\usepackage{spconf,graphicx}
\usepackage{hyperref}
\usepackage{amsmath, amssymb, mathtools, amsthm}
\usepackage{commath}
\usepackage{xcolor}
\usepackage{enumitem}
\usepackage{booktabs}
\usepackage{multicol}
\usepackage{float}
\usepackage{multirow}
\usepackage{colortbl}
\usepackage{enumitem}     % To remove spaces in itemize, etc

\DeclarePairedDelimiterX{\Paren}[1]{(}{)}{#1}
\DeclarePairedDelimiterX{\Brace}[1]{\{}{\}}{#1}
\DeclarePairedDelimiterX{\Brack}[1]{[}{]}{#1}
\DeclarePairedDelimiterX{\Abs}[1]{\rvert}{\lvert}{#1}
\DeclarePairedDelimiterX{\Norm}[1]{\lVert}{\rVert}{#1}
\DeclarePairedDelimiterX{\Avg}[1]{\langle}{\rangle}{#1}
\DeclarePairedDelimiterX{\Inner}[2]{\langle}{\rangle}{#1\,\delimsize\vert\,#2}
\newcommand*{\delimsize}{}
\DeclareMathOperator*{\argmin}{arg\,min}

\DeclareMathOperator{\Prox}{prox}

\newcommand*{\TransposeLetter}{\mathrm{\top}}
\newcommand*{\T}{^{\TransposeLetter}}

\newcommand*{\Vector}[1]{\boldsymbol{#1}}
\newcommand*{\Matrix}[1]{\mathbf{#1}}

\newcommand*{\HilbertSpace}[1]{\mathcal{#1}}

\newcommand*{\eg}{e.g., }
\newcommand*{\ie}{i.e., }

\newtheorem{theorem}{Theorem}

\title{MCN\lowercase{et}: Measurement-Consistent Networks via a Deep Implicit
Layer for Solving Inverse Problems}
\name{Rahul Mourya,\,\,
Jo\~ao F.\ C.\ Mota\thanks{Work supported by EPSRC's New Investigator Award
	(\href{https://gow.epsrc.ukri.org/NGBOViewGrant.aspx?GrantRef=EP/T026111/1}{EP/T026111/1}).}}
\address{Institute of Sensors, Signals, and Systems,\, Heriot-Watt
University,\, Edinburgh,\, UK}
\definecolor{darkredc}{RGB}{90,0,0}
\definecolor{darkgreen}{RGB}{0,90,0}
\definecolor{darkblue}{RGB}{0,0,90}			
\definecolor{thisblue}{rgb}{0,0,.4}
\definecolor{myred}{RGB}{153,0,0}
\definecolor{myblue}{RGB}{0,0,153}

\newcommand{\mypar}[1]{{\bf #1.}}

\begin{document}
\ninept
\setlength{\abovedisplayskip}{5pt}
\setlength{\belowdisplayskip}{5pt}
\linespread{0.9}
\setitemize{noitemsep,topsep=2pt,parsep=2pt,partopsep=2pt}

\maketitle
\begin{abstract}
	End-to-end deep neural networks (DNNs) have become the state-of-the-art (SOTA) for solving
	inverse problems. Despite their outstanding performance, during deployment,
	such networks are sensitive to minor variations in the testing pipeline and
	often fail to reconstruct small but important details, a feature critical in
	medical imaging, astronomy, or defence. Such instabilities in DNNs can be explained by
	the fact that they ignore the forward measurement model during deployment, and
	thus fail to enforce consistency between their output and the input
	measurements. To overcome this, we propose a framework
	that transforms any DNN for inverse problems into a measurement-consistent one.
	This is done by appending to it an implicit layer (or deep equilibrium network)
	designed to solve a model-based optimization problem. The implicit
	layer consists of a shallow learnable network that can be integrated into the end-to-end
	training while keeping the SOTA DNN fixed. Experiments on single-image super-resolution show that the proposed framework leads to significant improvements in reconstruction quality and robustness over the SOTA DNNs.
\end{abstract}
\section{Introduction}
\label{sec:intro}
\vspace{-1mm}
We consider the problem of reconstructing a vector $\Vector{x}^\star\in
\mathbb{R}^n$ of which we have only a few linear, noisy measurements
$\Vector{b} = \Matrix{A}\Vector{x}^\star + \Vector{\eta}$, where $\Matrix{A} \in
\mathbb{R}^{m\times n}$, with $m \ll n$, and $\Vector{\eta} \in \mathbb{R}^m$ is noise.
Such a linear inverse problem is challenging since, even when there is no
noise, the system $\Vector{b} = \Matrix{A}\Vector{x}$ contains infinitely many solutions, among
which the one we seek, $\Vector{x}^\star$. Classical approaches~\cite{Combettes2005,Beck2009,Mourya2015b} attempt to
reconstruct $\Vector{x}^\star$ by solving an optimization problem that balances
(or constrains) measurement fidelity and a regularizer encoding prior knowledge about $\Vector{x}^\star$, e.g., sparsity in a domain. A long-standing
challenge is to devise regularizers that, while computational tractable,
effectively describe classes of signals.

\begin{figure}[t]
	\centering
	\includegraphics[width=0.9\linewidth]{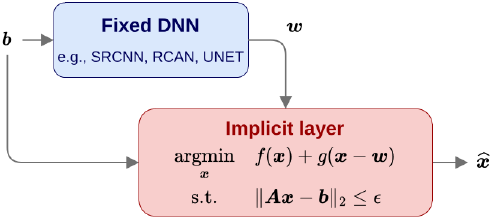}
	\vspace{-2mm}
	\caption{MCNet framework. The output $\Vector{w}$ of a fixed DNN is
		fed into an implicit layer designed to solve an inverse problem that ensures
		measurement consistency (via its constraint), fidelity to $\Vector{w}$ (via $g$), and 
		prior knowledge of $\Vector{x}$ (via $f$). $f$ and $g$ can be implemented as
		DNNs, and the entire pipeline is
		end-to-end trainable.}
	\label{fig:networkarchitecture}
	%\vspace*{-4mm}
\end{figure}
State-of-the-art (SOTA) methods for image reconstruction tasks such as
super-resolution, MRI reconstruction, or inpainting, leverage the power of deep neural networks
(DNNs) to directly learn a map from the measurements to the corresponding
reconstructed signal \cite{Nimisha2017,Dong2016,Qin19-ConvolutionalRecurrentNeuralNetworksForDynamicMRImageReconstruction,Sriram2020}. Despite their extraordinary performance, such end-to-end
schemes are unreliable, failing, for example, to detect small, unusual signals
absent in the training data~\cite{Antun2020}. Part of the reason is that during
deployment they ignore the forward measurement model, that is, $\Vector{b} =
\Matrix{A}\Vector{x} + \Vector{\eta}$, even when the model is
known~\cite{Vella2021b,Vella2021a}. Moreover, end-to-end DNNs for image
reconstruction require retraining whenever the measurement operator
$\Matrix{A}$ changes. Both shortcomings are absent in classical
optimization-based methods, albeit at the cost of poorer performance.

Recent work has thus tried to blend classical and DNN approaches to get the
benefits of both. We categorize such methodologies as 
deep unrolling networks (DUNets), plug-and-play methods (PnP),
post-processing (PP), and deep equilibrium models (DEQs). We briefly review
them in \S\ref{sec:relatedwork}.

\mypar{Our scheme}
The scheme we propose combines ideas from PnP, PP, and DEQs, and is represented
in Fig.~\ref{fig:networkarchitecture}. The measurement vector $\Vector{b}$
is passed through an end-to-end DNN, e.g., RCAN~\cite{Zhang2018} for
super-resolution or UNET \cite{Ronneberger2015} for MRI reconstruction. Its output 
$\Vector{w}$ will typically be close to $\Vector{x}^\star$ but
measurement inconsistent~\cite{Vella2021b}. To leverage the good quality of
$\Vector{w}$, the knowledge of the forward measurement model, and any other prior
information about $\Vector{x}^\star$, we pass $\Vector{b}$ and $\Vector{w}$ through an implicit layer that solves an optimization problem integrating all these elements. We refer to
the resulting scheme as \textit{measurement-consistent network} (MCNet), and
describe its architecture and training procedure in \S\ref{sec:sr_results}.
While MCNet can be applied to generic image reconstruction problems, this paper
focuses on single-image super-resolution. Because MCNet leverages a
high-quality DNN network, e.g., RCAN~\cite{Zhang2018}, its implicit layer
requires only a shallow network, usually with only $4-6$ layers, making it
simpler than alternative DEQ approaches~\cite{Gilton2021} or PnP
methods~\cite{Dong2018,Ryu2019}. Thus, it also requires less
training data than prior approaches.

\mypar{Contributions} 
Our contributions are as follows:
\begin{itemize}[leftmargin=*]
	\item We propose a framework, MCNet, for transforming any DNN for linear inverse
	problems into a measurement-consistent one.  
	\item By ensuring measurement-consistency, MCNet is more robust to small input perturbations than SOTA DNNs .
	\item We design an efficient training algorithm based on DEQ techniques.
	By interpreting an ADMM algorithm to solve the problem in the
	implicit layer as a fixed-point iteration, we leverage the implicit
	function theorem to train the network without ever unrolling
	the iterations. The number of iterations of the fixed-point algorithm can
	be modified during testing without significantly deteriorating the
	quality of the output. 
	\item We establish conditions under which the fixed-point algorithm 
	converges in the forward and backward passes.
	\item Extensive results on single-image super-resolution show
	that, with respect to prior methods, MCNet not only makes the DNN outputs
	measurement-consistent, but also yields significant gains in
	quality, despite having significantly fewer learnable parameters in the implicit layer.
\end{itemize}

\section{Related work}
\label{sec:relatedwork}

We briefly overview work that integrates DNNs into classical optimization-based
approaches.

\mypar{Deep unrolling networks (DUNets)}
First proposed in~\cite{Gregor2010}, DUNets are the current SOTA methods for MRI
reconstruction, low-dose CT, and emission tomography reconstruction. They are
designed by unrolling a few iterations of an iterative algorithm, e.g.,
forward-backward splitting~\cite{Combettes2005}, in which a proximal step is
replaced by a DNN. The number of unrolled iterations is typically small, $5$ to
$10$. And to reduce the memory requirements during training, parameters are
shared between layers. 
A consequence is that outputs from DUNets may not be
measurement-consistent~\cite{Vella2021a}. DUNets are also inflexible. For
example, the number of unrolled iterations has to be the same during training
and testing~\cite{Gilton2021}, and modifying the operator $\Matrix{A}$ requires
retraining from scratch.  

\mypar{Plug-and-play (PnP) methods}
PnP methods~\cite{Venkatakrishnan2013} modify model-based iterative algorithms,
e.g., forward-backward splitting or ADMM, by replacing one or more proximal
operators with nonconvex, SOTA denoisers like BM3D~\cite{Dabov2007} or
pre-trained DNNs, e.g., DNCNN~\cite{Zhang2017}. The resulting algorithms yield
excellent results on a variety of tasks, \eg  image deblurring, MRI, low-dose
CT reconstruction, and super-resolution~\cite{Dong2018}. To converge, however,
they require assumptions on the denoiser networks, e.g.,
contractivity~\cite{Ryu2019}. Unlike DUNets, PnP methods can handle different
operators $\Matrix{A}$ during training and testing. However, to achieve
reconstruction performance similar to end-to-end DNNs, they need to be trained
on very large datasets.

\mypar{Post-processing (PP)}
While PnP methods usually output measurement-consistent signals, other
strategies include designing data-consistency
layers~\cite{Aggarwal2019,Schempler18-ADeepCascadeOfConvolutionalNeuralNetworksForDynamicMRImageReconstruction,Qin19-ConvolutionalRecurrentNeuralNetworksForDynamicMRImageReconstruction}
or using cycle consistency
losses~\cite{Oh20-UnpairedDeepLearningForAcceleratedMRIUsingOptimalTransportDrivenCycleGAN}.
A particular effective strategy, however, is post-processing (PP) the output $\Vector{w}$ of an existing DNN
with a model-based optimization problem~\cite{Vella2019}:
\begin{align}
	\widehat{\Vector{x}}(\Vector{w}) 
	\in 
	\begin{array}[t]{cl}
		\underset{\Vector{x}}{\arg\min}  & \Norm{\Vector{x}}_\text{{TV}} + \beta
		\Norm{\Vector{x} - \Vector{w}}_{\text{TV}} \\
		\text{s.t.}  & \Matrix{A}\Vector{x} = \Vector{b}  \,,
	\end{array}
	\label{eq:tv_tv}
\end{align}
where $\Norm{\cdot}_{\text{TV}}$ is a total-variation (TV)
norm~\cite{Chambolle1997}, and  $\beta > 0$ a tradeoff parameter. Such PP
framework achieves better results than existing DNN on several tasks, including
image super-resolution \cite{Vella2021b}, MRI reconstruction
\cite{Vella2021a}, hyperspectral image super-resolution \cite{Vella2021c}.
Moreover, PP can handle slight variations in $\Matrix{A}$ \cite{Vella2021b}. 
A shortcoming, however, is the rigid selection of an explicit function, such as
the TV-norm, and the inability to train the network end-to-end. Our work builds
on PP by replacing the TV-norms in~\eqref{eq:tv_tv} with shallow
neural networks, and by leveraging techniques from deep equilibrium
networks to make the entire system end-to-end trainable.

\mypar{Deep equilibrium networks (DEQs)}
DEQs~\cite{Bai2019} achieve SOTA results on several vision
\cite{Bai2020,Gilton2021}, language \cite{Bai2019}, and graph \cite{Gu2020}
tasks. The output of a DEQ (or implicit) layer, rather than being defined explicitly as in
conventional DNNs, is defined implicitly. For example, if $\Vector{w}$
represents its input, the output of an implicit layer is the solution of
$h(\Vector{x}; \Vector{w}) = 0$, where $h$ is a generic differentiable map.
This generic model captures fixed-point equations, generalizing recurrent
backpropagation models \cite{Liao2018}, or optimality conditions of
optimization problems, effectively enabling the inclusion of hard-constraints
in DNNs~\cite{Amos17-OptNet-DifferentiableOptimizationAsALayerInNeuralNetworks}.
Unlike DUNets, the number of iterations to solve $h(\Vector{x}; \Vector{w}) =
0$ in the forward pass can be as large as required, without incurring any
additional memory footprint during the backward pass. This is because the
gradient can be propagated through the implicit layer via the implicit
function theorem. Thus, an implicit layer can be sandwiched
between conventional, explicit layers, and the whole network trained end-to-end. 

DEQs have been applied to solving inverse problems in~\cite{Gilton2021}. The
main difference with respect to our work is that they apply DEQ machinery directly to a
PnP method: ADMM solving a LASSO-type problem with the proximal operator
replaced by a DNN. Thus, they fail to leverage the power of end-to-end SOTA
DNNs. Our method, in contrast, leverages such methodologies and, as a result,
requires a significantly shallower DNN, with lower computational cost
during both training and testing, while achieving similar results, or better.
\vspace*{-10mm}

\section{\texorpdfstring{Proposed Framework: MCN\lowercase{et}}{Proposed Framework: MCNet}}
\label{sec:proposed_framework}
\vspace*{-4mm}
As depicted in Fig.~\ref{fig:networkarchitecture}, we propose to append an
implicit layer to the output of a fixed, pre-trained DNN, whose output is
$\Vector{w}$. The implicit layer solves a constrained optimization problem,
\begin{equation}
	\widehat{\Vector{x}}(\Vector{w}) \in 
	\argmin_{\Vector{x} \in \HilbertSpace{S}} f(\Vector{x}) + \beta g(\Vector{x}
	- \Vector{w}) \,,
	\label{eq:implicit_problem}
\end{equation}
where $\HilbertSpace{S} = \{\Vector{x}: \Norm{\Vector{b} - \Matrix{A}
	\Vector{x}}_{2} \leq \varepsilon \}$, for $\varepsilon \geq 0$, and $f, g\, :\,
\mathbb{R}^n \to \mathbb{R}$ are generic, possibly nonconvex, functions. We
estimate $\hat{\Vector{x}}$ in \eqref{eq:implicit_problem} by PnP-ADMM iterations similar to~\cite{Ryu2019} and replace the proximal operators of $f$ and/or $g$ by simple CNN denoisers. In this work, we consider simpler case and use $g(\cdot) = \Norm{\cdot}_2^2$, . The resulting algorithm can be interpreted as a fixed-point iteration.

\mypar{ADMM as a fixed-point iteration}
For simplicity, we consider the case $\varepsilon = 0$. Cloning the variable
$\Vector{x}$ into $\Vector{u}$ and setting $g(\cdot) = \Norm{\cdot}_2^2$, we rewrite~\eqref{eq:implicit_problem} as
\begin{equation}
	\label{eq:case1}
	\begin{array}[t]{ll}
		\underset{\Vector{x},\Vector{u}}{\text{minimize}} 
		&
		f(\Vector{u}) + \frac{\beta}{2} \Norm{\Vector{x} - \Vector{w}}^{2}_{2} + \iota_{\HilbertSpace{S}}(\Vector{x})
		\\
		\text{subject to}
		&
		\Vector{x} = \Vector{u}\,,
	\end{array}
\end{equation} 
where $\iota_{\HilbertSpace{S}}(s) = 0$ if $s \in \HilbertSpace{S}$, and
$\iota_{\HilbertSpace{S}}(s) = +\infty$ if $s \not\in \HilbertSpace{S}$, is the indicator function of $\HilbertSpace{S}$. 
Applying ADMM to~\eqref{eq:case1} results in
\vspace{1mm}
\begin{subequations}
	\label{eq:admminit}
	\begin{align}
		\Vector{u}^{(k+1)} &=\argmin_{\Vector{u}} f(\Vector{u}) + \frac{\rho}{2}
		\Norm{ \Vector{u} - \Vector{p}^{(k)}  }^{2}_{2} = \Prox_{\frac{f}{\rho}}
		\big( \Vector{p}^{(k)} \big)
		\label{eq:admminitu}
		\\
		\Vector{x}^{(k+1)} &=\argmin_{\Vector{x} \in \HilbertSpace{S}}\,\,\, 
		\frac{1}{2} \Norm{ \Vector{x}
			- \Vector{q}^{(k)}}^{2}_{2} = \Pi_{\HilbertSpace{S}} \big[\Vector{q}^{(k)} \big] 
		\label{eq:admminitx}
		\\
		\Vector{\lambda}^{(k+1)} & =  \Vector{\lambda}^{(k)} + \Vector{x}^{(k+1)} -
		\Vector{u}^{(k+1)}\,,
		\label{eq:admminitl}
	\end{align}
	\label{eq:admm_iters}
\end{subequations}
where $\rho > 0$ is the augmented Lagrangian parameter, $\Prox_{f}(\Vector{v}) :=
\argmin_{\Vector{x}}\,\, f(\Vector{x}) + \frac{1}{2}\Norm{\Vector{x} -
	\Vector{v}}_2^2$ the proximal operator of $f$ at $\Vector{v}$, and
$\Pi_{\HilbertSpace{S}}[\cdot]$ the projection operator (in this case,
it has a closed-form solution). We defined
$\Vector{p}^{(k)} := \big(\Vector{x}^{(k)} + \Vector{\lambda}^{(k)} \big)$ and
$\Vector{q}^{(k)} :=\big( {\beta \Vector{w} + \rho(\Vector{u}^{(k+1)} -
	\Vector{\lambda}^{(k)})} \big) / \big(\beta + \rho \big)$.

The key step is replacing $\Prox_{f/\rho}$ by a neural network $R_{\Vector{\theta}}$ with
parameters $\Vector{\theta}$. Performing this substitution and plugging
$\Vector{u}^{(k+1)}$ directly into  $\Vector{q}^{(k)}$ and~\eqref{eq:admminitl} yields:
\begin{subequations}
	\label{eq:fixedpointadmm}
	\begin{align}
		\Vector{x}^{(k+1)} 
		&= 
		\Pi_{\HilbertSpace{S}} 
		\Bigg[
		\frac
		{\beta \Vector{w} + \rho \Big( R_{\Vector{\theta}}\big(\Vector{x}^{(k)} +
			\Vector{\lambda}^{(k)}\big) - \Vector{\lambda}^{(k)}\Big)}
		{\beta + \rho} 
		\Bigg] 
		\label{subeq:admm_1_1} 
		\\
		\Vector{\lambda}^{(k+1)} 
		&= 
		\Vector{\lambda}^{(k)} 
		- R_{\Vector{\theta}} \big(\Vector{x}^{(k)} + \Vector{\lambda}^{(k)}\big)
		\notag
		\\
		&\qquad
		+ 
		\Pi_{\HilbertSpace{S}} 
		\Bigg[
		\frac
		{\beta \Vector{w} + \rho \Big( R_{\Vector{\theta}}\big(\Vector{x}^{(k)} +
			\Vector{\lambda}^{(k)}\big) - \Vector{\lambda}^{(k)}\Big)}
		{\beta + \rho} 
		\Bigg]\,.
		\label{subeq:admm_1_2}
	\end{align}
\end{subequations}
As the right-hand side of~\eqref{eq:fixedpointadmm} depends
only on variables indexed by $k$, i.e., $\Vector{x}^{(k)}$ and
$\Vector{\lambda}^{(k)}$, those equations can be interpreted as a fixed-point
iteration on the joint variable $\Vector{z} = \big(\Vector{x}, \Vector{\lambda}
\big)$. More compactly, they can be expressed as 
\begin{equation}
	\Vector{z}^{(k+1)} = F_{\Vector{\theta}}\Big(\Vector{z}^{(k)}; \Vector{w},
	\Matrix{A}, \Vector{b} \Big)\,,
	\label{eq:case1_fixed_point}
\end{equation}
where $F_{\Vector{\theta}}\,:\,\mathbb{R}^{n}\times\mathbb{R}^n \to
\mathbb{R}^n \times \mathbb{R}^n$ depends on the fixed parameters $\Vector{w}$,
$\Matrix{A}$, and $\Vector{b}$, and on the parameters
$\Vector{\theta}$ of the network $R_{\Vector{\theta}}$. If $F_{\Vector{\theta}}$ is contractive, then it has a
unique fixed-point $\Vector{z}^{(\infty)} = \big(\Vector{x}^{(\infty)},
\Vector{\lambda}^{(\infty)} \big)$, and the
iterations~\eqref{eq:case1_fixed_point} [equivalently, \eqref{eq:fixedpointadmm}]
converge to it. Below, we establish the conditions under which this holds.
A similar derivation applies to the case where $\varepsilon > 0$.

\mypar{Forward and backward passes}
In the forward-pass, if we directly use equations
\eqref{subeq:admm_1_1}-\eqref{subeq:admm_1_2} to compute the fixed-point of
\eqref{eq:case1_fixed_point}, convergence can be very slow. Instead, as
in~\cite{Bai2019}, we apply Anderson acceleration~\cite{Walker2011}, which uses
a few past iterations, typically 4-6, to identify promising directions to move
along in each iteration. In all our experiments, we set $\rho = 1$ and find that 150-200 iterations of Anderson acceleration achieve good accuracy. 

In the backward-pass, we need to compute the gradient of the loss $\ell$ with
respect to $\Vector{\theta}$, \ie $\partial \ell / \partial \Vector{\theta}$.
Following~\cite{Bai2019}, we leverage the fact that $\Vector{x}^{(\infty)} :=
\Vector{x}^{(\infty)}(\Vector{w},  \Matrix{A}, \Vector{b}; \Vector{\theta})$ is
a fixed-point of $F_{\Vector{\theta}}$, and calculate the required gradient
without unrolling the fixed-point iterations. Using the chain rule and 
the convention that $\partial/\partial{\cdot}$ represents a row vector, we write
${\partial \ell} / {\partial \Vector{\theta}}$ as a Jacobian-vector product:
\begin{equation}
	\label{eq:chain_rule}
	\frac{\partial \ell}{\partial \Vector{\theta}} = {\frac{\partial
			\Vector{x}^{(\infty)}}{\partial \Vector{\theta}}}\T \frac{\partial
		\ell}{\partial \Vector{x}^{(\infty)}}\,.
\end{equation}
To compute the Jacobian $\partial \Vector{x}^{(\infty)} / \partial
\Vector{\theta}$, we apply implicit differentiation to \eqref{eq:case1_fixed_point} at
the fixed-point $\Vector{z}^{(\infty)} = (\Vector{x}^{(\infty)}, 
\Vector{\lambda}^{(\infty)})$. After simplifying,
\begin{equation}
	\label{eq:implicit_differentiation}
	{\frac{\partial \Vector{x}^{(\infty)}}{\partial  \Vector{\theta}}} = \Bigg[
	\Matrix{I} - {\frac{\partial  F_{\Vector{\theta}}(\Vector{x};
			\cdot)}{\partial \Vector{x}}} \Bigg|_{\Vector{x}^{(\infty)}} \Bigg]^{-1} \Bigg[
	\frac{\partial F_{\Vector{\theta}}(\Vector{x}^{(\infty)}; \cdot)}{\partial
		\Vector{\theta}} \Bigg]\,.
\end{equation}
Plugging  \eqref{eq:implicit_differentiation} into \eqref{eq:chain_rule} yields:
\begin{equation}
	\label{eq:inv_jacb_vec}
	\frac{\partial \ell}{\partial \Vector{\theta}} = \Bigg[ {\frac{\partial
			F_{\theta}(\Vector{x}^{(\infty)}; \cdot)}{\partial \Vector{\theta}}}\Bigg]\T
	\underbrace{
		\Bigg[ \Matrix{I} - {\frac{\partial  F_{\Vector{\theta}}(\Vector{x};
				\cdot)}{\partial \Vector{x}}} \Bigg|_{\Vector{x}^{(\infty)}} \Bigg]^{-\top}
		\frac{\partial \ell}{\partial \Vector{x}^{(\infty)}}
	}_{=: \Vector{s}^{(\infty)}}\,,
\end{equation}
where $\Matrix{I}$ is the identity matrix. 
The quantity $\Vector{s}^{(\infty)}$ in \eqref{eq:inv_jacb_vec}  can be
computed as the fixed-point of
\begin{equation}
	\Vector{s}^{(k+1)} = 
	\Bigg[ {\frac{\partial F_{\Vector{\theta}}(\Vector{x}; \cdot)}{\partial
			\Vector{x}}} \Bigg|_{\Vector{x}^{(\infty)}} \Bigg]\T \Vector{s}^{(k)} +
	\frac{\partial \ell}{\partial \Vector{x}^{(\infty)}}\,.
	\label{eq:grad_fixed_point}
\end{equation}
We again use Anderson acceleration to compute this fixed-point.
This can be done efficiently with recent auto-differentiation tools that compute
the Jacobian-vector product in \eqref{eq:grad_fixed_point}. Once
$\Vector{s}^{(\infty)}$ is computed, we replace it
in~\eqref{eq:inv_jacb_vec} and obtain $\partial \ell / \partial
\Vector{\theta}$. In all our training, we find that 50-80 iterations of
Anderson acceleration are enough to achieve good accuracy in both the forward and backward passes.

\mypar{Convergence of fixed-point iterations}
Fixed-point theory guarantees that the iterations \eqref{eq:case1_fixed_point}
and \eqref{eq:grad_fixed_point} converge to a unique fixed-point if the
associated iteration maps are contractive. The following result slightly
modifies the argument in~\cite{Ryu2019} to show that $F_{\Vector{\theta}}$ is
contractive.

\begin{theorem}
	If the network $R_{\Vector{\theta}}$ is contractive, then the map
	$F_{\Vector{\theta}}$ is contractive, and the iterations
	\eqref{eq:case1_fixed_point} converge to a unique fixed-point.
\end{theorem}
\begin{proof}
	As in \cite{Ryu2019}, the ADMM iterations in the forward pass are equivalent to
	Douglas-Rachford splitting iterations, which can be written as
	$\Vector{y}^{(k+1)} = T(\Vector{y}^{(k)})$ with
	\begin{equation}
		T = \frac{1}{2} \Matrix{I} + \frac{1}{2} \big(2 {\Pi}_{\HilbertSpace{S}} -
		\Matrix{I} \big) \Big(2 \Prox_{\frac{f}{\rho}} - \Matrix{I} \Big) \,.
		\label{eq:drs_fixed_point}
	\end{equation}
	The term $\big(2 {\Pi}_{\HilbertSpace{S}} - \Matrix{I} \big)$ is non-expansive,
	since ${\Pi}_{\HilbertSpace{S}}$ is a projection onto a nonempty, closed, and
	convex set $\HilbertSpace{S}$. If we assume $\big(2 \Prox_{f/\rho} -
	\Matrix{I} \big)$ is contractive, then $T$ becomes contractive as well. This is
	the case if we replace $\Prox_{f/\rho}$ by a contractive network
	$R_{\Vector{\theta}}$.  
\end{proof}

The proof differs from the one in~\cite{Ryu2019} in that 
$\Pi_{\HilbertSpace{S}}$ in their case is a contractive map. Thus, we need to
assume the last term of~\eqref{eq:drs_fixed_point} is contractive.
To make $R_{\Vector{\theta}}$ contractive, as in
\cite{Ryu2019}, we constrain the Lipschitz constant of $R_{\Vector{\theta}}$
by spectral normalization (SN) of its weights. In particular, we consider
$R_{\Vector{\theta}}$ to be a 6-layer CNN with 1-Lipschitz ReLU activation
functions, and apply SN to each convolutional layer, bounding their spectral norm below 1.

Likewise, \eqref{eq:grad_fixed_point} converges to a unique
fixed-point whenever the Jacobian $\partial_{\Vector{x}}
F_{\Vector{\theta}}(\Vector{x}; \cdot)$, evaluated at any $\Vector{x}$, has a
spectral norm strictly below 1. As shown in \cite{Gilton2021}, this is
equivalent to $F_{\Vector{\theta}}$ being contractive. Thus, a contractive
$R_{\Vector{\theta}}$ also ensures the convergence of
\eqref{eq:grad_fixed_point}.

\section{Experimental Results}
\label{sec:sr_results}
\vspace{-1mm}
To illustrate the effectiveness of the proposed MCNet, we focused on one imaging
application: single-image super-resolution.

\mypar{Experimental setup}
As backbone networks for MCNet, we selected the SOTA DNNs SRCNN \cite{Dong2016},
RDN \cite{Zhang2018residual} and RCAN \cite{Zhang2018}. SRCNN was designed for
single channel images, while RDN and RCAN support RGB color images. To make
performance assessment uniform across different methods, we designed MCNet to
operate on a single channel: Y from the YCbCr colorspace. Thus, MCNet takes in
only the Y channel from the outputs of RDN and RCAN. We used pretrained models of
these DNNs, obtained from the respective authors, and keep them fixed during the experiments. We also
included super-resolved (SR) outputs obtained by bicubic interpolation, TV-TV
minimization \cite{Vella2021b}, and a version of the PnP-ADMM method that
solves \eqref{eq:case1} with $\beta = 0$, and uses a DNCNN \cite{Ryu2019} with
17-layers as a denoiser. We tested all methods at scales 2, 3, and 4, on the
benchmark datasets Set5~\cite{Bevilacqua2012} and Set14~\cite{Zeyde2012}. 
Low-resolution (LR) inputs are generated from original high-resolution (HR)
images by bicubic interpolation. And we assess the quality of the algorithms'
outputs by computing PSNR and SSIM values on the Y channel image.

\mypar{Training MCNet}
We selected $R_{\Vector{\theta}}$ in MCNet as a DNCNN \cite{Ryu2019} with
6-layers, but without biases in the convolution layer and no batch
normalization. We will refer to it as DNCNN6. We observed that increasing the
number of layers significantly increases the computational cost, but does not
translate into observable gains in image quality. To train MCNet, we first
pretrained $R_{\Vector{\theta}}$ using 4 different DNCNN6 networks at different
levels of white Gaussian noise, and selected the one that performed the best on
Set5 with the PnP-ADMM method mentioned above. Using the selected DNCNN6, we
then searched for a good initial value for parameter $\beta$ using a coarse grid,
by solving \eqref{eq:case1} with the ADMM iterations in \eqref{eq:admm_iters}.
This pretraining of $R_{\Vector{\theta}}$ and $\beta$ enabled training MCNet
successfully in a few number of epochs, $40$. For training, we followed the
methodology in~\cite{Ryu2019}. In particular, the contractiveness of
$R_{\Vector{\theta}}$ was ensured by applying SN to all layers. MCNet was
trained on image patches of the dataset T91 and validated on Set5. 
Adam algorithm \cite{Kingma2014} was applied with an initial learning-rate
$10^{-4}$, which was reduced to $10^{-5}$. Finally, we mention that MCNet+SRCNN
(i.e., MCNet using SRCNN as a backbone)
was trained with $\varepsilon = 0$, whereas MCNet+RDN and MCNet+RCAN were
trained with $\varepsilon = 0.1$ [cf.\ the definition of $\HilbertSpace{S}$
after~\eqref{eq:implicit_problem}]. 
In fact, we observed that decreasing 
$\varepsilon$ slightly decreases the performance of MCNet. Thus,
there is a trade-off between ensuring measurement consistency and obtaining
good image quality.
Other details on the training and code to replicate
our experiments will be available
online.\footnote{https://github.com/mouryarahul/SR-MCNet}.

\mypar{Results}
Tables \ref{table:consistency} and \ref{table:quality_performance} present the
results of our experiments, with our method being denoted as MCNet+Fixed-DNN. 
Table \ref{table:consistency} shows the average data fidelity,
$\Norm{\Matrix{A}\widehat{\Vector{x}} - \Vector{b}}$, of several methods
evaluated on Set5. It can be observed that methods based on constrained
optimization, including MCNet, have consistency values much better than SOTA,
end-to-end DNNs. Specifically, MCNet+SRCNN, which uses $\varepsilon = 0$,
yields a consistency value akin to TV-TV and PnP, whereas when we relax
$\varepsilon$ to $0.1$, consistency values are similar to
end-to-end networks.% (SRCNN, RDN, and RCAN).

Better measurement consistency, however, does not always translate into better
image quality, as shown in Table \ref{table:quality_performance}, which depicts
PSNR and SSIM values. For example, for $\times 4$ scaling (3rd block), PnP is
outperformed by SRCNN. Yet, the table shows that, in all but one case, MCNet
outperforms all the other methods, sometimes by a large margin.  

\begin{table}[!t]
	\small
	\caption{Average data-fidelity $\Norm{\Matrix{A} \widehat{ \Vector{x}} -
			\Vector{b}}_{2}$, where $\widehat{\Vector{x}}$ is the output
		of an algorithm, and $\Vector{b}$ the input. Values for $\times 2$
		scaling on dataset Set5. The smaller the values, the better. Similar values are obtained on dataset Set14.}
	\centering
	\begin{tabular}{@{}lll@{}}
		\specialrule{1pt}{1pt}{1pt}
		\textbf{Method} & \textbf{Scale} & $\Norm{\Matrix{A} \widehat{ \Vector{x}} - \Vector{b}}_{2}$ \\ 
		\midrule
		Bicubic & $\times 2$ & $1.2652\times 10^{0}$  \\ %\hline
		PnP & $\times 2$ & $2.2167\times 10^{-5}$ \\ %\hline
		SRCNN\cite{Dong2016} & $\times 2$ & $3.0626\times 10^{-1}$ \\ %\hline
		RDN\cite{Zhang2018residual} & $\times 2$ & $8.3431\times 10^{-2}$  \\ %\hline
		RCAN\cite{Zhang2018} & $\times 2$ & $8.5996\times 10^{-2}$ \\ %\hline
		TVTV \cite{Vella2021b}+SRCNN & $\times 2$ & $7.4337\times 10^{-6}$ \\ %\hline
		TVTV+RDN & $\times 2$ & $7.0079\times 10^{-6}$\\ %\hline
		TVTV+RCAN & $\times 2$ & $7.1031\times 10^{-6}$ \\ %\hline
		\textbf{\textcolor{myred}{MCNet+SRCNN}} & $\times 2$ & $2.6276\times 10^{-5}$ \\ %\hline
		\textbf{\textcolor{myred}{MCNet+RDN}} & $\times 2$ & $4.1483\times 10^{-1}$\\ %\hline
		\textbf{\textcolor{myred}{MCNet+RCAN}} & $\times 2$ & $4.4271\times 10^{-1}$ \\ 
		\specialrule{1pt}{1pt}{1pt}
	\end{tabular}
	\vspace{-4mm}
	\label{table:consistency}
\end{table}

\begin{table}[!t]
	\small
	\caption{Average PSNR (SSIM) in dB on datasets Set5 and Set14, for $\times
		2$, $\times 3$, and $\times 4$ scaling. The higher the values, the
		better. The table compares 11 different methods and best values are highlighted.}
	\centering
	\begin{tabular}{@{}llll@{}}
		\specialrule{1pt}{1pt}{1pt}
		\textbf{Method} & \textbf{Scale} & \textbf{Set5} & \textbf{Set14} \\ \midrule
		Bicubic & $\times 2$ & 33.6890 (0.9375) & 30.2327 (0.8834) \\ %\hline
		PnP & $\times 2$ & \textbf{37.4337 (0.9814)} & {32.8681} (0.9496) \\ %\hline
		SRCNN & $\times 2$ & 36.6530 (0.9597) & 32.5932 (0.9169) \\ %\hline
		TVTV+SRCNN & $\times 2$ & 36.7813 (0.9607) & 32.7086 (0.9185) \\ %\hline
		\textbf{\textcolor{myred}{MCNet+SRCNN}} 
		& $\times 2$ & {37.2272 (0.9623)} & \textbf{32.8980 (0.9195)} \\ \hline
		
		Bicubic & $\times 3$ & 30.8867 (0.8941) & 27.8792 (0.8110) \\ %\hline
		PnP & $\times 3$ & 33.6468 ({0.9650}) & 29.9157 (0.9077) \\ %\hline
		SRCNN & $\times 3$ & 33.3614 (0.9323) & 29.9630 (0.8540) \\ %\hline
		TVTV+SRCNN & $\times 3$ & 33.4805 (0.9343) & 30.0659 (0.8569) \\ %\hline
		\textbf{\textcolor{myred}{MCNet+SRCNN}}
		& $\times 3$ & \textbf{34.0458 (0.9401)} & \textbf{30.2655 (0.8604)} \\ \hline
		
		Bicubic & $\times 4$ & 28.4337 (0.8234) & 25.9724 (0.7260) \\ %\hline
		PnP & $\times 4$ & 30.2903 (0.9262) & 27.4729 (0.8539) \\ %\hline
		SRCNN & $\times 4$ & 30.3270 (0.8726) & 27.6146 (0.7724) \\ %\hline
		TVTV+SRCNN & $\times 4$ & 30.4829 (0.8771) & 27.7241 (0.7770) \\ %\hline
		\textbf{\textcolor{myred}{MCNet+SRCNN}}
		& $\times 4$ & \textbf{31.0690 (0.8888)} & \textbf{27.9799 (0.7839)} \\ 
		\specialrule{1pt}{1pt}{1pt}
		
		RDN & $\times 2$ & 38.0434 (0.9657) & 33.8416 (0.9277) \\ %\hline
		TVTV+RDN & $\times 2$ & 38.0449 (0.9657) & 33.8351 (0.9278) \\ %\hline
		\textbf{\textcolor{myred}{MCNet+RDN}}
		& $\times 2$ & \textbf{38.1246 (0.9657)} & \textbf{33.8765  (0.9277)} \\ \hline
		
		RDN & $\times 3$ & 35.2322 (0.9487) & 31.1424 (0.8749) \\ %\hline
		TVTV+RDN & $\times 3$ & 35.2339 (0.9487) & 31.1510 (0.8749) \\ %\hline
		\textbf{\textcolor{myred}{MCNet+RDN}}
		& $\times 3$ & \textbf{35.2808 (0.9486)} & \textbf{31.1903 (0.8749)} \\ \hline
		
		RDN & $\times 4$ & 32.3311 (0.9069) & 28.8365 (0.8050) \\ %\hline
		TVTV+RDN & $\times 4$ & 32.3399 (0.9069) & 28.8407 (0.8047) \\ %\hline
		\textbf{\textcolor{myred}{MCNet+RDN}}
		& $\times 4$ & \textbf{32.3452 (0.9068)} & \textbf{28.8458 (0.8050)} \\
		\specialrule{1pt}{1pt}{1pt}
		
		RCAN & $\times 2$ & 38.1996 (0.9662) & 34.2138 (0.9304) \\ %\hline
		TVTV+RCAN & $\times 2$ & 38.2041 (0.9662) & 34.1978 (0.9304) \\ %\hline
		\textbf{\textcolor{myred}{MCNet+RCAN}}
		& $\times 2$ & \textbf{38.2838 (0.9662)} & \textbf{34.2512 (0.9304)} \\ \hline
		
		RCAN & $\times 3$ & 35.4484 (0.9501) & 31.3028 (0.8774) \\ %\hline
		TVTV+RCAN & $\times 3$ & 35.4548 (0.9501) & 31.3090 (0.8772) \\ %\hline
		\textbf{\textcolor{myred}{MCNet+RCAN}}
		& $\times 3$ & \textbf{35.5024 (0.9500)} & \textbf{31.3463 (0.8773)} \\ \hline
		
		RCAN & $\times 4$ & 32.6554 (0.9099) & 28.9400 (0.8076) \\ %\hline
		TVTV+RCAN & $\times 4$ & \textbf{32.6655 (0.9099)} & 28.9579 (0.8073) \\ %\hline
		\textbf{\textcolor{myred}{MCNet+RCAN}}
		& $\times 4$ & 32.6653 (0.9094) & \textbf{28.9782 (0.8073)} \\ 
		\specialrule{1pt}{1pt}{1pt}
	\end{tabular}
	\vspace{-3mm}
	\label{table:quality_performance}
\end{table}

\section{Conclusions}
\label{sec:conclusion}
We proposed a framework to transform any existing DNN for solving inverse
problems into a measurement-consistent network, when the acquisition model is known precisely. 
The core of the framework is an implicit layer defined by a constrained
optimization problem, which is solved by an ADMM algorithm with learnable
parameters. Extensive experimental results on a single-image super-resolution
problem show superior performance with respect to end-to-end DNNs, plug-and-play, and
other methods. Although the proposed method requires solving an optimization
problem during training and inference, the additional computational cost is
justified in applications that require precise fidelity to measurements, for
example, in areas in
medicine, astronomy or defence. Future work includes exploring different
applications and different algorithmic choices, e.g., using a shallow CNN
instead of an $\ell_2$-norm to integrate the output of the fixed-DNN into our
method.

% To start a new column (but not a new page) and help balance the last-page
% column length use \vfill\pagebreak.
% -------------------------------------------------------------------------
%\vfill\pagebreak
%\clearpage

% References should be produced using the bibtex program from suitable
% BiBTeX files (here: strings, refs, manuals). The IEEEbib.bst bibliography
% style file from IEEE produces unsorted bibliography list.
% -------------------------------------------------------------------------
\small
\bibliographystyle{IEEEbib}
\bibliography{references}

\end{document}